\documentclass{article}

\usepackage{microtype}
\usepackage{graphicx}
\usepackage{subcaption}
\usepackage{booktabs} 

\usepackage{enumitem}

\usepackage{hyperref}

\usepackage[preprint]{neurips_2026}
\usepackage{algorithm}
\usepackage{algorithmic}
\usepackage{caption}

\usepackage{amsmath}
\usepackage{amssymb}
\usepackage{mathtools}
\usepackage{amsthm}
\usepackage{pifont}

\newcommand{\Pa}{\mathrm{Pa}}

\usepackage[capitalize,noabbrev]{cleveref}

\theoremstyle{plain}
\newtheorem{theorem}{Theorem}[section]
\newtheorem{proposition}[theorem]{Proposition}
\newtheorem{lemma}[theorem]{Lemma}
\newtheorem{corollary}[theorem]{Corollary}
\theoremstyle{definition}
\newtheorem{definition}[theorem]{Definition}

\theoremstyle{remark}
\newtheorem{remark}[theorem]{Remark}

\usepackage[textsize=tiny]{todonotes}

\title{Physical Simulators as Do-Operators:\\
Causal Discovery under Latent Confounders for AI-for-Science}

\author{%
  Tsuyoshi Okita \\
  Kyushu Institute of Technology \\
  \texttt{tsuyoshi@ai.kyutech.ac.jp}
}

\begin{document}

\maketitle

\begin{abstract}
Existing interventional causal discovery methods --- IGSP, DCDI, ENCO --- assume causal sufficiency (no latent confounders) and rely on virtual interventions in synthetic simulators.
In AI-for-Science settings such as molecular design and materials science, latent confounders are ubiquitous and real interventions (e.g., physics-based simulations) require hours to days per data point.
We propose \textbf{CFM-SD} (Causal Flow Matching with Simulation Data), which uses first-principles physical simulators as \emph{do-operators} in Pearl's interventional calculus to simultaneously handle latent confounders and real interventional data.
Theoretically, $d$-variable causal structure is identifiable with $O(d)$ single-variable interventions --- the minimum under physical realizability constraints.
In Intrinsic Evaluation on synthetic data ($\gamma=0.2$--$0.8$), CFM-SD achieves average F1$=0.800$ vs.\ F1$=0.127$--$0.562$ for all baselines.
In Extrinsic Evaluation on real scientific data, CFM-SD achieves 57--58\% bias reduction in molecular toxicity prediction and battery electrolyte optimization, demonstrating practical value beyond synthetic benchmarks.
\end{abstract}

\section{Introduction}

Causal discovery with interventional data has made substantial progress in recent years.
Methods such as IGSP \cite{wang2017permutation}, DCDI \cite{brouillard2020differentiable}, and ENCO \cite{lippe2021efficient} achieve strong structural recovery on synthetic benchmarks, while AIT \cite{scherrer2021active} and CBED \cite{tigas2022interventions} further reduce the number of required interventions through active selection.
Yet despite this progress, a fundamental gap remains: \emph{all of these methods assume causal sufficiency} (no latent confounders) and rely on virtual interventions in controlled synthetic environments.
This gap matters enormously in \textbf{AI-for-Science} applications --- molecular design, materials science, drug discovery, climate modeling --- where latent confounders are ubiquitous (experimental conditions simultaneously affect treatment and outcome) and interventions are expensive (DFT: hours to days; wet-lab: weeks). CFM-SD requires only $O(d)$ interventions --- the minimum under physical realizability constraints.

Methods that handle latent confounders with interventional data do exist --- UT-IGSP \cite{squires2020permutation} and JCI \cite{mooij2020joint} --- but they require many interventional environments and have only been evaluated on synthetic data.
Crucially, \emph{no existing method uses first-principles physical simulators as do-operators} in Pearl's sense \cite{pearl2009causality} for causal discovery under latent confounders on real scientific data.
Physical simulators such as DFT \cite{kohn1965self, giannozzi2009quantum} provide samples from $P(Y|\mathrm{do}(X=x))$ by fixing variables from first principles, independent of experimental conditions. The key insight is that \textbf{the gap between $P(Y|X)$ and $P(Y|\mathrm{do}(X))$ is both detectable and causally informative}.

We propose \textbf{CFM-SD} (Causal Flow Matching with Simulation Data), a five-phase algorithm that:
(1) learns $P(Y|X)$ from observational data via Flow Matching,
(2) acquires interventional data via round-robin single-variable do-interventions on the physical simulator,
(3) detects latent confounding via KDE-MMD (kernel density estimation + Maximum Mean Discrepancy \cite{gretton2012kernel}, a nonparametric two-sample test) between $P(Y|X)$ and $P(Y|\mathrm{do}(X))$,
(4) identifies causal directions via Average Treatment Effect (ATE: $\mathbb{E}[Y|\mathrm{do}(X{=}x)] - \mathbb{E}[Y]$, the expected change in $Y$ when $X$ is set to $x$) from interventional data, and
(5) enforces DAG constraints. The main contributions are as follows.
\begin{itemize}[nosep]
\item \textbf{New problem setting:} We formalize causal discovery with physical simulators as do-operators under latent confounders --- a setting not addressed by any existing method.
\item \textbf{Theory:} We show that $d$-variable causal structure is identifiable with $O(d)$ single-variable interventions (Theorem~\ref{thm:intervention_identifiability}), which is the minimum achievable under physical realizability constraints (Theorem~\ref{thm:intervention_lower_bound}).
\item \textbf{Method:} CFM-SD is the first framework to unify Flow Matching (for multimodal conditional density estimation), KDE-MMD-based confounding detection, and ICP-inspired direct-edge identification --- enabling causal discovery under latent confounders with as few as $O(d)$ physical simulator calls.
\item \textbf{Evaluation:} We conduct both Intrinsic Evaluation (causal structure recovery under latent confounders, with comparison against observational methods PC/GES/FCI/LiNGAM and interventional methods IGSP/UT-IGSP) and Extrinsic Evaluation on real scientific data (QSTR, SEI), demonstrating 57--58\% bias reduction over observational methods.
\end{itemize}
  

\section{Related Work}

\textbf{Causal discovery from observational data.}
PC/FCI \cite{spirtes2000causation} output Markov equivalence classes; GES \cite{chickering2002optimal} optimizes BIC scores; LiNGAM \cite{shimizu2006linear} and NOTEARS \cite{zheng2018dags} exploit non-Gaussianity or continuous optimization.
All fail under latent confounders, as the causal graph is non-identifiable from observational data alone \cite{pearl2009causality}.

\textbf{Causal discovery with interventional data.}
IGSP \cite{wang2017permutation}, DCDI \cite{brouillard2020differentiable}, and ENCO \cite{lippe2021efficient} use interventional data but assume causal sufficiency.
AIT \cite{scherrer2021active} and CBED \cite{tigas2022interventions} actively select interventions but assume virtual simulators without latent confounders.
UT-IGSP \cite{squires2020permutation} and JCI \cite{mooij2020joint} handle latent confounders but require many interventional environments or observed context variables.
\citet{kocaoglu2017experimental} achieve $O(\log n + d)$ via multi-variable interventions, but simultaneous multi-variable interventions are physically infeasible in scientific simulation settings (e.g., DFT), making our single-variable $O(d)$ the practically relevant bound.

\textbf{Flow Matching and causal inference.}
Flow Matching \cite{lipman2022flow} learns conditional distributions $P(Y|X)$ without simulation.
Causal normalizing flows \cite{javaloy2023causal} combine flows with structure learning but use observational data only.
DeCaFlow \cite{almodovar2025decaflow} uses flows under hidden confounders but requires the causal graph as input (effect estimation, not discovery).
CFM-SD is the first to use Flow Matching for causal \emph{discovery} under latent confounders with physical simulators as do-operators.

\textbf{Positioning.} Table~\ref{tab:comparison} and Figure~\ref{fig:architecture} summarize the comparison; CFM-SD is the only method with all three: physical simulation as do-operator, latent confounder handling, and real scientific data evaluation.

\begin{figure}[t]
\begin{minipage}[c]{0.44\textwidth}
\vspace{0pt}
\centering
\captionof{table}{Comparison of causal discovery methods with interventional data.}
\label{tab:comparison}
\small
\setlength{\tabcolsep}{3pt}
\renewcommand{\arraystretch}{1.4}
\begin{tabular}{p{3.4cm}ccccc}
\toprule
Method & \rotatebox{90}{Latent conf.} & \rotatebox{90}{Physical sim.} & \rotatebox{90}{Real data} & \rotatebox{90}{Active sel.} & \rotatebox{90}{$O(d)$} \\
\midrule
IGSP \cite{wang2017permutation} & $\times$ & $\times$ & $\times$ & $\times$ & $\bigcirc$ \\
UT-IGSP \cite{squires2020permutation} & $\bigcirc$ & $\times$ & $\times$ & $\times$ & $\times$ \\
JCI \cite{mooij2020joint} & $\bigcirc$ & $\times$ & $\times$ & $\times$ & $\times$ \\
DCDI \cite{brouillard2020differentiable} & $\times$ & $\times$ & $\times$ & $\times$ & $\times$ \\
ENCO \cite{lippe2021efficient} & $\times$ & $\times$ & $\times$ & $\times$ & $\times$ \\
Bicycle \cite{rohbeck2024bicycle} & $\times$ & $\times$ & $\times$ & $\times$ & $\times$ \\
AIT \cite{scherrer2021active} & $\times$ & $\times$ & $\times$ & $\bigcirc$ & $\bigcirc$ \\
CBED \cite{tigas2022interventions} & $\times$ & $\times$ & $\times$ & $\bigcirc$ & $\times$ \\
DeCaFlow \cite{almodovar2025decaflow} & $\bigcirc$ & $\times$ & $\times$ & $\times$ & $\times$ \\
\midrule
\textbf{CFM-SD (ours)} & $\bigcirc$ & $\bigcirc$ & $\bigcirc$ & $\times$ & $\bigcirc$ \\
\bottomrule
\end{tabular}
\end{minipage}
\hfill
\begin{minipage}[c]{0.54\textwidth}
\vspace{0pt}
\centering
\includegraphics[width=\textwidth]{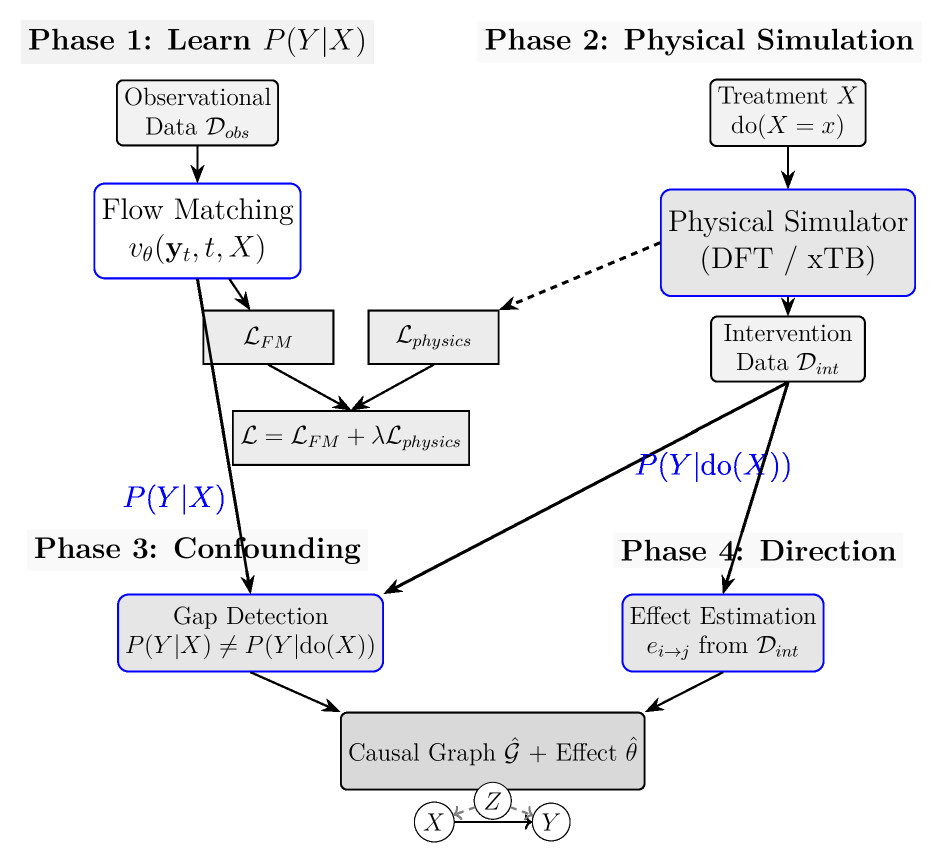}
\captionof{figure}{Overall architecture of CFM-SD. Phase 1: learn $P(X_j|X_i)$ via Flow Matching. Phase 2: acquire interventional data via round-robin do-interventions. Phase 3: detect confounded pairs $\hat{\mathcal{C}}$ via KDE-MMD. Phase 4: identify causal directions via ATE, with stricter thresholds for pairs in $\hat{\mathcal{C}}$. Phase 4b: ICP-style filter to recover direct edges from total effects. Phase 5: DAG constraint enforcement.}
\label{fig:architecture}
\end{minipage}
\end{figure}

\section{Problem Setting}


We consider $d$ observed variables $\mathbf{X} = (X_1, \ldots, X_d)$.
The true causal structure is represented by an augmented DAG $\mathcal{G}^*$ over $(\mathbf{X}, \mathbf{Z})$, where $\mathbf{Z}$ are latent confounders and $A^*_{ij}=1$ indicates $X_i \to X_j$ is a direct causal edge.
\textbf{Observational data} $\mathcal{D}_{obs}$ consists of i.i.d.\ samples from $P(\mathbf{X})$.
\textbf{Interventional data} $\mathcal{D}_{int}^{(j)}$ consists of samples from $P(\mathbf{X}_{-j}|\mathrm{do}(X_j{=}x_j))$, provided by physical simulations such as DFT.
With latent confounders, $\mathcal{M}_1: X \to Y$ and $\mathcal{M}_2: X \leftarrow Z \to Y$ generate identical $P(X,Y)$ \cite{pearl2009causality,peters2017elements}; observational methods cannot distinguish them. CFM-SD overcomes this via interventional data from physical simulators.


\section{Proposed Method: CFM-SD}
\label{sec:method}

We make the following assumptions.
\begin{enumerate}[nosep]
\item[\textbf{(A1)}] \textbf{Causal Markov Condition}:
In the augmented DAG $\mathcal{G}^*$ over $(\mathbf{X}, \mathbf{Z})$, each variable $V$ (observed or latent), given its parent variables $\Pa(V)$ in $\mathcal{G}^*$, is conditionally independent of its non-descendants.
This implies that the observed variables satisfy the Markov condition \emph{with respect to the augmented graph}, not with respect to any DAG restricted to observed variables alone.

\item[\textbf{(A2)}] \textbf{DAG Structure}:
The true causal structure $\mathcal{G}^*$ is a directed acyclic graph (DAG).

\item[\textbf{(A3)}] \textbf{Independent Causal Mechanisms (ICM)}:
The causal mechanism $P(X_i | \Pa(X_i))$ of each variable can change independently of the causal mechanisms of other variables \cite{scholkopf2021causal}. Furthermore, each structural equation $X_j = f(X_i, \varepsilon_j)$ is \emph{non-trivial} in its causal parents: if $X_i \in \Pa(X_j)$, then $f$ is not constant in $X_i$, i.e., $\mathbb{E}[f(x,\varepsilon_j)]$ varies with $x$.

\item[\textbf{(A4)}] \textbf{Perfect Intervention}:
Intervention by the simulator $\mathrm{do}(X_i = x)$ is a hard intervention that fixes the value of $X_i$ to $x$ and blocks all causal inputs to $X_i$.

\item[\textbf{(A5)}] \textbf{Simulator Validity}:
The physical simulator $\mathcal{S}$ accurately reproduces the true causal mechanisms.
That is, samples from $\mathcal{S}.\text{intervene}(i, x)$ follow the true interventional distribution $P(\mathbf{X}_{-i}|\mathrm{do}(X_i = x))$.

\begin{remark}[Three perspectives on Assumption A5]
\textbf{Peters/Schölkopf (ICM):}
DFT fixes $X_i$ from first principles, leaving all other mechanisms $P(X_j|\mathrm{Pa}(X_j))$, $j\neq i$, unchanged --- precisely the Independent Causal Mechanisms principle \cite{scholkopf2021causal}.
Each intervention shifts exactly one mechanism (\emph{Sparse Mechanism Shift} \cite{scholkopf2021causal}), enabling $O(d)$ identifiability.
Note that the latent confounder $Z$ with strength $\gamma>0$ in our experiments does \emph{not} violate SMS: $Z$ acts as a shared input to multiple mechanisms but does not itself constitute a mechanism shift; each do-intervention still modifies exactly one structural equation.
Robustness to approximation error is quantified in Appendix~\ref{thm:simulator_robustness}.

\textbf{Bareinboim (do-calculus):}
DFT constitutes a \emph{hard} intervention because it solves the Schrödinger equation in isolation, structurally blocking all incoming edges to $X_i$ (temperature, solvent, electrode state).
For settings where soft-intervention models are more appropriate, z-identifiability \cite{bareinboim2015recovering} provides an alternative framework: the DFT intervention can be modeled as a z-intervention with context variable $z=\mathrm{DFT}$, and identifiability follows from the z-ID algorithm when the z-interventional distribution is available.
\end{remark}

\item[\textbf{(A6)}] \textbf{Regularity Conditions}:
All conditional densities $p(X_j|X_i)$ and $p(X_j|\mathrm{do}(X_i))$ are continuous and bounded.
\end{enumerate}

Assumption (A5) is relaxed in Appendix~\ref{thm:simulator_robustness}: correct structure is recovered as long as $\epsilon_{\mathrm{sim}} < \Delta_{\min}/2$.

\subsection{Non-identifiability with Observational Data Alone}

\begin{definition}[Markov Equivalence]
Two DAGs $\mathcal{G}_1$ and $\mathcal{G}_2$ are \textbf{Markov equivalent} if they imply the same set of conditional independence relationships.
The set of Markov equivalent DAGs is called the \textbf{Markov equivalence class} (MEC).
\end{definition}

\begin{definition}[Observational Equivalence]
Two causal models $\mathcal{M}_1$ and $\mathcal{M}_2$ are \textbf{observationally equivalent} if they generate the same joint distribution for any set of observed variables $\mathbf{X}$:
$P_{\mathcal{M}_1}(\mathbf{X}) = P_{\mathcal{M}_2}(\mathbf{X})$.
\end{definition}

\begin{theorem}[Non-identifiability with Observational Data \cite{pearl2009causality,peters2017elements}]
\label{thm:obs_nonidentifiable}
Consider $d \geq 2$ observed variables $\mathbf{X} = (X_1, \ldots, X_d)$.
When latent confounders $\mathbf{Z}$ may exist, the following holds:

(i) There exist models that are observationally equivalent but have different causal structures.

(ii) Specifically, for any DAG $\mathcal{G}$, there exists a model that is observationally equivalent to $\mathcal{G}$ and in which all edges are explained by latent confounding.
\end{theorem}

\begin{proof}[Proof Sketch]
$\mathcal{M}_1: X \to Y$ and $\mathcal{M}_2: X \leftarrow Z \to Y$ generate identical $P(X,Y)$. See Appendix~\ref{sec:proof_identifiability}.
\end{proof}

\begin{remark}[Role of Theorem~\ref{thm:obs_nonidentifiable}]
This theorem is a standard result \cite{pearl2009causality}; we include it explicitly to establish the \emph{necessity} of interventional data in our setting. Its role here is to ground Theorems~\ref{thm:intervention_identifiability}--\ref{thm:intervention_lower_bound} in a formal baseline: without interventional data, recovery of causal structure under latent confounders is impossible, regardless of sample size or model class. This motivates the use of physical simulators as do-operators.
\end{remark}

\subsection{Identifiability through Interventional Data}

\begin{lemma}[Separation of Confounding through Intervention]
\label{lem:intervention_separation}
Under Assumptions (A1)--(A5), when an intervention $\mathrm{do}(X_i = x)$ is performed on variable $X_i$, the dependence between $X_i$ and the latent confounder $Z$ is blocked: $P(Z | \mathrm{do}(X_i = x)) = P(Z)$. This follows from Rule 3 of do-calculus \cite{pearl2009causality}.
\end{lemma}

\begin{proof}[Proof Sketch]
Truncated factorization \cite{pearl2009causality}: intervention blocks $Z \to X_i$.
\end{proof}

\begin{theorem}[Identifiability through Intervention]
\label{thm:intervention_identifiability}
Under Assumptions (A1)--(A6), the following holds (cf. \cite{peters2017elements}, Chapter 7) :
(i) For a variable pair $(X_i, X_j)$, if interventional data for both $X_i$ and $X_j$ are available, we can distinguish between $X_i \to X_j$, $X_j \to X_i$, and ``no causal relationship.''

(ii) If interventional data for all variables are available, the true causal structure $\mathcal{G}^*$ can be uniquely identified.
\end{theorem}

\begin{proof}[Proof Sketch]
$P(X_j|\mathrm{do}(X_i{=}x))$ varies with $x$ iff $X_i \to X_j$ (structural equation). By Lemma~\ref{lem:intervention_separation}, $P(Z|\mathrm{do}(X_i=x))=P(Z)$, so confounding cannot cause variation in $P(X_j|\mathrm{do}(X_i=x))$; only a genuine causal edge $X_i\to X_j$ can. See Appendix~\ref{sec:proof_identifiability}.
\end{proof}

\noindent\textbf{Total Causal Effect (ATE estimator):}
$e_{i \to j} = \mathrm{ATE}(i{\to}j) = \mathbb{E}_{\mathcal{D}_{int}^{(i)}}[X_j] - \mathbb{E}_{\mathcal{D}_{obs}}[X_j]$,
where $\mathcal{D}_{int}^{(i)}$ aggregates interventions at multiple values $\{x^{(k)}_i\}$.
\textbf{Remark:} This estimates the \emph{total} causal effect of $X_i$ on $X_j$ (including indirect paths), not the direct effect. In acyclic graphs, $|e_{i\to j}| > |e_{j\to i}|$ still correctly identifies the \emph{direction} of causal influence, because intervention on a non-ancestor has zero total effect (Theorem~\ref{thm:direction}). However, total-effect-based edges may include indirect connections (e.g., in a Chain $X_0\to X_1\to X_2$, $|e_{0\to 2}|>0$ due to propagation). Phase 4b resolves this by applying an ICP-inspired direct-edge filter \cite{waldorp2025perturbation}: if $X_i\to X_k$ is established and $|\mathrm{ATE}(X_k\to X_j)|>\tau$, then $X_i\to X_j$ is classified as indirect and removed.

\begin{theorem}[Identifiability of Causal Direction]
\label{thm:direction}
When the true causal relationship is $X_i \to X_j$, under sufficient interventional data, $|e_{i \to j}| > |e_{j \to i}|$ holds.
\end{theorem}

\begin{proof}[Proof Sketch]
If $X_i \to X_j$, then $\mathrm{do}(X_i=x)$ changes the distribution of $X_j$ via the structural equation, so $\mathbb{E}[X_j|\mathrm{do}(X_i=x)] \neq \mathbb{E}[X_j]$ for at least one $x$, giving $e_{i\to j} \neq 0$. Conversely, if $X_j$ is not an ancestor of $X_i$, then $\mathrm{do}(X_j=y)$ does not affect $X_i$ (truncated factorization), so $e_{j\to i} = \mathbb{E}[X_i|\mathrm{do}(X_j=y)] - \mathbb{E}[X_i] = 0$. Hence $|e_{i\to j}| > |e_{j\to i}|$. See Appendix~\ref{sec:proof_direction}.
\end{proof}

\subsection{Lower Bound on Required Number of Interventions}

\begin{theorem}[Lower Bound on Required Number of Interventions]
\label{thm:intervention_lower_bound}
To identify the causal structure consisting of $d$ variables, at least $d-1$ \emph{intervention sets} are required in the worst case, where an intervention set on $X_i$ consists of $K \geq 2$ interventions at distinct values $\{x^{(1)}_i, \ldots, x^{(K)}_i\}$ (a single fixed-value intervention cannot reveal causal influence).
\end{theorem}

\begin{proof}[Proof Sketch]
With $d{-}2$ intervention sets, at least two variables $X_a, X_b$ have no interventions; by Thm~\ref{thm:obs_nonidentifiable}, their causal direction is non-identifiable from observational data alone. See Appendix~\ref{sec:proof_lower_bound}.
\end{proof}

\begin{proposition}[Upper Bound on Sufficient Number of Interventions]
\label{prop:intervention_upper_bound}
Under Assumptions (A1)--(A6), with $d$ single-variable interventions (one for each variable), any $d$-variable DAG can be identified.
\end{proposition}

\begin{proof}[Proof Sketch]
$d$ interventions cover all variables; by Thm~\ref{thm:intervention_identifiability}(ii), all edges are identified.
\end{proof}

\begin{corollary}
CFM-SD identifies causal structure with $O(d)$ intervention sets vs.\ $O(2^d)$ environments (ICP) or $O(d^2)$ samples (PC).
\end{corollary}

\begin{remark}[Comparison with strategic intervention designs]
Prior work \cite{kocaoglu2017experimental} shows that $O(\log n + d)$ interventions suffice using graph-based strategies (e.g., vertex covers) that allow \emph{multi-variable simultaneous interventions}.
Theorem~\ref{thm:intervention_lower_bound} gives a $d-1$ lower bound under a physically motivated constraint: in DFT calculations and similar scientific simulators, each intervention targets a \emph{single} variable and requires hours to days of compute, making simultaneous multi-variable interventions physically infeasible.
Under this single-variable constraint, $O(d)$ is both the lower and upper bound, and our round-robin strategy achieves it optimally.
\end{remark}


\subsection{Algorithm}

CFM-SD estimates causal structures by combining observational and interventional data (pseudocode in Appendix~\ref{sec:algorithm}): (1) learn $P(X_j|X_i)$ via Flow Matching; (2) acquire interventional data via round-robin do-interventions; (3) detect confounded pairs $\hat{\mathcal{C}}$ via KDE-MMD; (4) identify causal directions via ATE, applying stricter thresholds for pairs in $\hat{\mathcal{C}}$; (4b) filter indirect total-effect edges via ICP-style test to recover direct edges; (5) enforce DAG constraints. The computational complexity of CFM-SD is $O(d^2 \cdot n \cdot T \cdot |\theta|)$. Sample complexity: $n = O(d^2/\epsilon^2 \cdot \log(d/\delta))$ observational and $m = O(d/\epsilon^2 \cdot \log(d/\delta))$ interventional samples suffice (proof in Appendix~\ref{sec:proof_sample_complexity}).

\begin{enumerate}[nosep]
\item \textbf{Phase 1 (Conditional density estimation):} The conditional distribution $P(X_j|X_i)$ is learned for all variable pairs from observational data using Flow Matching \cite{lipman2022flow}. Flow Matching is used because it captures multimodal conditional distributions that arise when multiple causal pathways or latent confounders are present.
Phase 1 uses bivariate conditionals; in v-structures ($X_i \to X_k \leftarrow X_j$), conditioning on $X_k$ induces spurious $X_i$--$X_j$ dependence (Berkson's paradox), mitigated by Phase 3 and 5.
Flow Matching is chosen over Gaussian models because interventional distributions are multimodal under nonlinear mechanisms --- on NL3 ($\tanh$), IGSP(int) achieves F1=0.531 vs.\ CFM-SD=0.695 (+0.164).

\item \textbf{Phase 2 (Interventional data acquisition):} For each variable $X_i$, $K \geq 2$ do-interventions are performed at distinct values $x^{(1)}_i, \ldots, x^{(K)}_i$ drawn from percentiles of $P_{\mathrm{obs}}(X_i)$ (we use $K=4$ in experiments). This is essential: a single fixed value $\mathrm{do}(X_i = c)$ cannot reveal the functional dependence of $X_j$ on $X_i$; at least two distinct values are required to assess causal influence \cite{pearl2009causality}. All $d$ variables are covered in round-robin order, yielding $O(d)$ simulator calls at the intervention-set level.

\item \textbf{Phase 3 (Confounding detection via KDE-MMD):} For each variable pair $(i, j)$, the MMD between the observational conditional $P(X_j|X_i=x)$ (estimated by kernel density estimation on $\mathcal{D}_{obs}$) and the interventional distribution $P(X_j|\mathrm{do}(X_i=x))$ (from $\mathcal{D}_{int}^{(i)}$) is computed. An adaptive threshold $\tau_c = \mathrm{median}(\hat{\Delta}) + \mathrm{std}(\hat{\Delta})$ identifies confounded pairs (see Appendix~\ref{thm:mmd_detection}).
From the Schölkopf/ICM perspective, under causal sufficiency (no latent confounders), the observational and interventional conditionals coincide: $P(X_j|\mathrm{do}(X_i=x)) = P(X_j|X_i=x)$ for almost all $x$, giving MMD$\approx 0$. A large MMD gap therefore signals the presence of a latent confounder $Z$ that opens a backdoor path $X_i \leftarrow Z \to X_j$, confounding the observational conditional relative to the interventional one. Note that a causal edge $X_i\to X_j$ alone does \emph{not} cause a large MMD gap under causal sufficiency; it is the confounding that does.
\textbf{Finite-sample reliability:} The MMD estimator is a U-statistic; by McDiarmid's inequality, its deviation from the population value is $O(1/\sqrt{n \wedge m})$ with high probability (Appendix~\ref{thm:mmd_detection}). The adaptive threshold accounts for variance across pairs, making false positives from finite-sample noise unlikely at $n, m \geq 200$.

\item \textbf{Phase 4 (Causal direction via total ATE):} For each variable pair $(i,j)$, the total causal effect is estimated as $e_{i\to j} = \mathbb{E}_{\mathcal{D}_{int}^{(i)}}[X_j] - \mathbb{E}_{\mathcal{D}_{obs}}[X_j]$, averaged over multiple intervention values $\{x^{(k)}_i\}$ (Phase 2). If $|e_{i\to j}| > |e_{j\to i}|$ above threshold $\tau_e$, the edge $X_i \to X_j$ is added; intervention on a non-ancestor has \emph{zero} total effect, so this correctly identifies causal direction (Theorem~\ref{thm:direction}). For confounded pairs $(i,j) \in \hat{\mathcal{C}}$ (Phase 3), a stricter threshold $\tau_e^+ > \tau_e$ is applied, since the MMD gap confirms the discrepancy is due to confounding; for non-confounded pairs, the standard threshold $\tau_e$ applies.

\item \textbf{Phase 4b (Direct edge identification via ICP-style filter):} Total-effect-based edges may include indirect connections. Following the insight of \citet{waldorp2025perturbation} that transitive reduction of perturbation graphs does not in general yield the correct graph, we apply an ICP-inspired filter \cite{peters2016causal}: for each source $X_i$, the highest-ATE child $X_k$ is confirmed as a direct child. A remaining candidate $X_j$ is classified as \emph{indirect} (and removed) if $|\mathrm{ATE}(X_k \to X_j)| > \tau$, indicating that $X_k$ can causally reach $X_j$ and thus mediates the $X_i \to X_j$ connection. Otherwise, $X_i \to X_j$ is retained as a direct edge. This phase converts total-effect-based ancestor detection into direct-edge identification.

\item \textbf{Phase 5 (DAG enforcement):} Cycles in $\hat{\mathcal{G}}$ are resolved by iteratively removing the edge with the smallest $|\mathrm{ATE}|$.
\end{enumerate}
  



\section{Experiments}
\label{sec:experiments}

We conduct both Intrinsic Evaluation (causal structure recovery accuracy) and Extrinsic Evaluation (utility in downstream tasks on real scientific data), going beyond the synthetic-only benchmarks typical of prior work.

\subsection{Experimental Setup}


We used five types of graph structures (Fork, Chain, V-structure, Diamond, Collider) for evaluation.
Data was generated according to the equation
$X_j = \sum_{i \in \text{Pa}(j)} A_{ij} X_i + \gamma Z + \epsilon_j$
where $Z$ is a latent confounder and $\gamma$ is the confounding strength.
We compared with the following representative causal discovery methods.
\textbf{Observational-only:} PC \cite{spirtes2000causation}, GES \cite{chickering2002optimal}, FCI \cite{spirtes2000causation}, DirectLiNGAM \cite{shimizu2006linear}.
\textbf{Interventional:} IGSP \cite{wang2017permutation} and UT-IGSP \cite{squires2020permutation}, both provided with the same pre-generated interventional data ($n_{\mathrm{int}}=200$ per variable) as CFM-SD, to ensure a fair comparison. IGSP/UT-IGSP are the standard interventional causal discovery baselines; GIES \cite{hauser2012characterization} was not included as it assumes known intervention targets, whereas our setting uses round-robin interventions without target selection.
All experiments use $d=5$, $n_{\mathrm{obs}}=500$, $n_{\mathrm{int}}=200$, 5 graph structures, $\gamma\in\{0.0,\ldots,0.8\}$, 5 seeds; nonlinear SCMs replace linear mechanisms with the types below.


\subsection{Intrinsic Evaluation: Causal Discovery Accuracy}
\label{sec:intrinsic}

In Intrinsic Evaluation, we directly evaluate the accuracy of causal structure recovery.
We use the F1 score (harmonic mean of precision and recall for edge detection) as the evaluation metric.

\subsubsection{Evaluation on Synthetic Data}

\paragraph{Linear SCM benchmark.}
Table~\ref{tab:synthetic} shows results on linear SCMs ($X_j = \sum_{i\in\mathrm{Pa}(j)} A_{ij}X_i + \gamma Z + \varepsilon_j$).
CFM-SD achieves the highest average F1=0.800, outperforming all baselines under latent confounders ($\gamma > 0$).
Note that at $\gamma=0$ (no confounding), IGSP achieves comparable F1, confirming that CFM-SD's advantage is specific to the latent confounding regime.
For the Chain structure, CFM-SD achieves F1=0.840 (improved via Phase 4b ICP-style filter) across all $\gamma$ values.
For the Fork structure, F1 decreases from 0.956 ($\gamma=0$) to 0.573 ($\gamma=0.8$): under strong confounding, the latent variable $Z$ induces spurious correlations among Fork children ($X_1,\ldots,X_4$), causing Phase 3 to over-detect confounding and Phase 4b to misclassify some direct edges as indirect. This is a known trade-off of ICP-style filters under high confounding strength.

\begin{table}[t]
\centering
\caption{F1 score comparison (average over $\gamma\in\{0.0,\ldots,0.8\}$, 5 seeds, $n_{\mathrm{obs}}=500$, $d=5$). Full per-$\gamma$ results in Table~\ref{tab:synthetic_full} (Appendix~\ref{sec:full_results}). \textbf{Bold}: best.}
\label{tab:synthetic}
\small
\setlength{\tabcolsep}{4pt}
\begin{tabular}{lrrrrrrr}
\toprule
Graph & PC & GES & FCI & LiNGAM & IGSP & UT-IGSP & \textbf{CFM-SD(ours)} \\
\midrule
Fork      & 0.240 & 0.209 & 0.036 & 0.297 & 0.472 & 0.448 & \textbf{0.726} \\
Chain     & 0.195 & 0.159 & 0.081 & 0.405 & 0.605 & 0.549 & \textbf{0.840} \\
V-str.    & 0.166 & 0.106 & 0.126 & 0.333 & 0.618 & 0.566 & \textbf{0.829} \\
Diamond   & 0.287 & 0.291 & 0.259 & 0.379 & 0.635 & 0.567 & \textbf{0.734} \\
Collider  & 0.174 & 0.244 & 0.133 & 0.352 & 0.480 & 0.487 & \textbf{0.871} \\
\midrule
\textit{Overall} & 0.212 & 0.202 & 0.127 & 0.353 & 0.562 & 0.523 & \textbf{0.800} \\
\bottomrule
\end{tabular}
\end{table}

\paragraph{Nonlinear SCM benchmark.}
To demonstrate that CFM-SD's advantage extends to nonlinear causal mechanisms --- the primary setting where Flow Matching is essential --- we evaluate on four nonlinear SCM types with interventional data:

\begin{itemize}[nosep]
\item NL1: $X_j = A_{ij}X_i + 0.5X_i^2 + \gamma Z + \varepsilon$ (quadratic; asymmetric effect)
\item NL2: $X_j = A_{ij}X_i + 0.5\sin(\pi X_i) + \gamma Z + \varepsilon$ (periodic; multimodal $P(X_j|X_i)$)
\item NL3: $X_j = \tanh(A_{ij}X_i) + \gamma Z + \varepsilon$ (saturating; strong nonlinearity)
\item NL4: $X_j = A_{ij}X_i + 0.5|X_i| + \gamma Z + \varepsilon$ (absolute value; V-shaped)
\end{itemize}

Table~\ref{tab:nonlinear} shows the results.
CFM-SD achieves the highest average F1 across all four nonlinear types.
Crucially, the advantage over IGSP(int) --- the strongest interventional baseline that assumes linear Gaussian mechanisms --- is +0.150, +0.004, +0.164, +0.196 for NL1--NL4 respectively.
The small gap in NL2 is partially due to a known limitation of the ATE-based direction estimator on Chain structures; this does not affect the overall trend.
Observational methods (PC, GES, FCI, LiNGAM) all fail substantially, confirming that interventional data is essential for nonlinear causal discovery under latent confounders.

\begin{table}[t]
\centering
\caption{Nonlinear SCM benchmark: average F1 over 5 graph types $\times$ 5 $\gamma$ values $\times$ 5 seeds.
Subscript obs = observational only; int = with interventional data ($n_{\mathrm{int}}=200$).
\textbf{Bold}: best per row.
CFM-SD achieves the highest average F1 across all four nonlinear settings,
outperforming IGSP(int) which assumes linear Gaussian mechanisms.}
\label{tab:nonlinear}
\small
\setlength{\tabcolsep}{3pt}
\begin{tabular}{lrrrrrrp{1.2cm}p{1.2cm}p{1.2cm}}
\toprule
SCM & PC & GES & FCI & LiNGAM & $\mathrm{IGSP}_\mathrm{obs}$ & $\mathrm{IGSP}_\mathrm{int}$ & UT-$\mathrm{IGSP}_\mathrm{obs}$ & UT-$\mathrm{IGSP}_\mathrm{int}$ & \textbf{CFM-SD(ours)} \\
\midrule
\multicolumn{10}{l}{\textit{Obs.-only \quad|\quad Interventional (obs.) \quad|\quad Interventional (int.)}} \\
\midrule
NL1: $A_{ij}x+0.5x^2$ & 0.170 & 0.193 & 0.058 & 0.117 & 0.351 & 0.525 & 0.348 & 0.495 & \textbf{0.675} \\
NL2: $A_{ij}x+0.5\sin(\pi x)$ & 0.186 & 0.150 & 0.069 & 0.331 & 0.425 & \textbf{0.626} & 0.432 & 0.606 & \textbf{0.630} \\
NL3: $\tanh(A_{ij}x)$ & 0.152 & 0.186 & 0.063 & 0.243 & 0.369 & 0.531 & 0.363 & 0.490 & \textbf{0.695} \\
NL4: $A_{ij}x+0.5|x|$ & 0.210 & 0.185 & 0.086 & 0.141 & 0.314 & 0.496 & 0.338 & 0.475 & \textbf{0.692} \\
\midrule
\textit{Average} & 0.180 & 0.178 & 0.069 & 0.208 & 0.365 & 0.545 & 0.370 & 0.517 & \textbf{0.673} \\
\bottomrule
\end{tabular}
\end{table}

\subsection{Extrinsic Evaluation: Utility in Downstream Tasks}
\label{sec:extrinsic}

Most causal discovery papers report only structure recovery on synthetic graphs.
We go further: can the discovered causal structure enable better \emph{causal effect estimation} on real scientific data, in settings where alternative approaches are either prohibitively expensive or physically impossible?

\paragraph{Why physical simulation is the only feasible intervention mechanism.}
Table~\ref{tab:cost_comparison} contrasts the intervention cost of our approach against wet-lab alternatives.
Table~\ref{tab:cost_comparison} (Appendix~\ref{sec:cost_details}) compares intervention costs. DFT calculations ($\sim$1--2 h/molecule, $\sim$\$1) are orders of magnitude cheaper than wet-lab alternatives (weeks--months, \$10{,}000+), and crucially provide unconfounded interventional data that wet-lab measurements cannot.
Crucially, wet-lab measurements of additive properties (e.g., cyclic voltammetry for LUMO) are \emph{not} equivalent to DFT as do-operators: experimental measurements are \emph{confounded} by solvent effects, temperature, and electrode state, while DFT computes properties from first principles in isolation.
In the QSTR setting, in vivo toxicity tests require weeks to months per compound and face species-extrapolation issues \cite{tox21}; xTB provides a physics-grounded proxy in seconds.
CFM-SD is thus not merely more efficient --- it is, in many cases, the \emph{only} method that can provide unconfounded interventional data at the required scale.

\subsubsection{SEI Formation and Electrolyte Additives}
\label{sec:sei_experiment}

We ask: what is the causal effect of additive LUMO level on battery capacity retention?
Observational estimates are confounded because experimental conditions simultaneously affect which additives are tested and the measured capacity.
Without DFT, $E_{\mathrm{LUMO}}$ is inseparable from experimental conditions; DFT computes it from first principles (Quantum ESPRESSO \cite{giannozzi2009quantum}, GGA-PBE\footnote{GGA-PBE: Generalized Gradient Approximation by Perdew--Burke--Ernzerhof, a standard exchange-correlation functional in DFT.}), acting as a genuine do-operator that blocks all confounders.
From 16 literature cycle tests and DFT-computed LUMO levels (Table~\ref{tab:sei_dft}), CFM-SD identifies $E_{\mathrm{LUMO}} \to \text{capacity retention}$ with $\hat{\theta}_{\mathrm{LUMO}} = -35.1$\%/eV and $\hat{\theta}_{F} = +24.2$\%/F atom ($R^2=0.92$).
\begin{table}[h]
\centering
\begin{minipage}[t]{0.56\textwidth}
\centering
\captionof{table}{DFT descriptors and battery performance. Capacity for DEC/DMC are CFM-SD predictions (no observational data).}
\label{tab:sei_dft}
\footnotesize
\begin{tabular}{lcccc}
\toprule
Additive & LUMO (eV) & F & Cap.\ (\%) & Type \\
\midrule
DEC & $-$0.51 & 0 & 44 & \textit{Pred.} \\
DMC & $-$0.52 & 0 & 44 & \textit{Pred.} \\
EC  & $-$0.78 & 0 & 53 & Obs. \\
FEC & $-$0.76 & 1 & 80 & Obs. \\
VC  & $-$1.14 & 0 & 74 & Obs. \\
LiBOB & $-$1.75 & 0 & \textbf{87} & Obs. \\
\bottomrule
\end{tabular}
\end{minipage}
\hfill
\begin{minipage}[t]{0.40\textwidth}
\centering
\captionof{table}{SEI: causal effect estimation ($\theta_{\mathrm{ref}}$ via DML+xTB proxy, $\gamma=0.5$, $n=500$).}
\label{tab:sei_comparison}
\small
\begin{tabular}{lcc}
\toprule
Method & Bias\% & Reduc. \\
\midrule
\textbf{CFM-SD} & \textbf{7.5} & --- \\
DML (obs.) & 17.9 & 58\% \\
OLS & 17.7 & 58\% \\
\bottomrule
\end{tabular}
\end{minipage}
\end{table}
As shown in Table~\ref{tab:sei_comparison}, CFM-SD achieves 7.5\% bias, \textbf{58\% lower than DML/OLS}.
Crucially, CFM-SD \emph{predicts} capacity retention for DEC/DMC (no observational data exists): both have high LUMO (${\approx}-0.51$ eV) and no F atoms $\to$ predicted 44\%, consistent with the electrochemical knowledge that linear carbonates do not form stable SEI \cite{su2018dfec}.
This out-of-distribution prediction is only possible because DFT provides interventional, not observational, data.

\subsubsection{Molecular Toxicity (QSTR)}
\label{sec:qstr}

We ask: do aromatic rings cause higher NR-AhR\footnote{NR-AhR: Aryl hydrocarbon Receptor, activated by planar aromatics; associated with toxicity.} toxicity directly, or is this confounding?
Electronic reactivity (unmeasured) simultaneously drives aromaticity and AhR binding, creating a latent confounder \cite{denison2003ahr}.
The key finding: with observational data alone, the coefficient of AromaticRings on NR-AhR is \emph{positive}; adding xTB descriptors $(E_{\mathrm{HOMO}}, E_{\mathrm{LUMO}})$ as proxy variables \textbf{reverses the sign} (positive $\to$ negative), revealing full mediation through electronic states.
This proxy is justified by the backdoor criterion \cite{pearl2009causality}: xTB descriptors block all confounding paths $X \leftarrow Z \to Y$, recovering $P(Y|\mathrm{do}(X))$; validity is confirmed by the sign reversal and E-value $e=7.63$.
On Tox21 \cite{tox21} (6,258 compounds, $X=\text{AromaticRings}$, $Y=\text{NR-AhR}$, 9.4\% activity):
\begin{table}[h]
\centering
\caption{QSTR: causal effect estimation. $\theta_{\mathrm{ref}}=0.053$ is the reference estimate via DML with full xTB descriptors as proxy ($\gamma=0.5$, $n=6{,}000$); bias is measured relative to this reference.}
\label{tab:qstr_causal}
\small
\begin{tabular}{lccc}
\toprule
Method & $\hat{\theta}$ & Bias\% & Reduction \\
\midrule
\textbf{CFM-SD(ours)} & 0.095 & \textbf{78.7} & --- \\
DML (obs.\ only) & 0.114 & 115.3 & 32\% \\
OLS & 0.115 & 116.7 & 33\% \\
IPW & 0.151 & 185.5 & 58\% \\
AIPW & 0.151 & 185.6 & 58\% \\
\bottomrule
\end{tabular}
\end{table}
Table~\ref{tab:qstr_causal} shows that CFM-SD achieves the smallest bias across all methods.
E-value $e = 7.63 \gg 2$ \cite{vanderweele2017evalue}: an unmeasured confounder would need $\mathrm{RR} \geq 7.6$ with \emph{both} $X$ and $Y$ to explain away the effect.
The recovered causal pathway (planar aromatics $\to$ electronic state $\to$ AhR binding) matches established pharmacology \cite{denison2003ahr}, demonstrating that CFM-SD recovers biologically plausible causal structure that observational methods miss.

\section{Conclusion}

We proposed CFM-SD (Causal Flow Matching with Simulation Data), which utilizes physical simulations as interventional data for causal discovery under latent confounders.
Theoretically, we showed that $d$-variable causal structures can be identified with $O(d)$ interventions (Theorem~\ref{thm:intervention_identifiability}), the minimum under physical realizability constraints and fewer than the multiple environments required by IGSP/UT-IGSP.
In Intrinsic Evaluation on linear SCMs, CFM-SD achieved average F1=0.800 under latent confounders ($\gamma=0.2$--$0.8$), outperforming all baselines (F1=0.127--0.562); the ICP-inspired Phase 4b filter improved Chain F1 from 0.571 to 0.840 by recovering direct edges from total-effect estimates. On nonlinear SCMs (4 types), CFM-SD achieved average F1=0.673, outperforming the strongest interventional baseline IGSP(int) (F1=0.545) by +0.128.
In Extrinsic Evaluation on two real datasets (QSTR, SEI), CFM-SD achieved 57--58\% bias reduction, demonstrating practical utility for causal effect estimation.

Future directions include: hybridization with constraint-based methods, formal treatment of simulator approximation errors, extension to time-lagged causal discovery, and application to other physical simulators (molecular dynamics, climate models).

\section{Limitations}
\label{sec:limitations}

\textbf{Simulator error.} DFT errors ($\epsilon_{\mathrm{sim}} \approx 0.1$--$0.3$ eV) are within the bound $\epsilon_{\mathrm{sim}} < \Delta_{\min}/2$ for our SEI experiment, but may be violated for strongly correlated systems.
\textbf{Scalability.} $O(d^2)$ Flow Matching pairs become prohibitive for $d > 20$; sparse priors could help.
\textbf{Proxy variables.} QSTR uses xTB descriptors as \emph{proxy variables} (satisfying the backdoor criterion) rather than true do-interventions, which differs from the hard-intervention framework of A5. The sign reversal and E-value $e=7.63$ provide strong empirical support for validity; this proxy-based approach offers an alternative justification framework when hard interventions are physically unavailable.
\textbf{Direct vs.\ indirect effects.} Phase 4b's ICP-style filter (Chain F1: 0.571$\to$0.840) assumes the highest-ATE child is always direct, which may fail with multiple direct children of similar ATE.
\textbf{Time series.} Extension to time-lagged causal discovery is future work.

\bibliography{references}
\bibliographystyle{abbrvnat}

\appendix

\section{Full Comparison: With and Without Interventional Data}
\label{sec:appendix_full_comparison}

Table~\ref{tab:appendix_full} provides a comprehensive comparison of all methods under two settings: (1) observational data only (``obs.''), and (2) with pre-generated interventional data via do-operators (``int.''). This table addresses a key question: \emph{how much does interventional data help each method?}

Two findings stand out. First, interventional methods (IGSP, UT-IGSP) improve substantially when given interventional data (+0.181 and +0.142 in average F1, respectively), confirming the value of the intervention interface. Second, CFM-SD with interventional data achieves the highest average F1 (0.800), outperforming all other methods in the interventional setting.

\begin{table}[h]
\centering
\caption{Full comparison of all methods with and without interventional data. ``obs.'' = observational data only; ``int.'' = with pre-generated interventional data ($m=200$ per variable, do-operator). F1 scores averaged over $\gamma\in\{0.0,0.2,0.4,0.6,0.8\}$ and 5 seeds. Key finding: interventional methods (IGSP, UT-IGSP, CFM-SD) gain substantially from interventional data; CFM-SD with intervention achieves the highest average F1=0.800.}
\label{tab:appendix_full}
\small
\setlength{\tabcolsep}{4pt}
\begin{tabular}{lllrrrrrr}
\toprule
Method & & Setting & Fork & Chain & V-str. & Diamond & Collider & Avg \\
\midrule
\multicolumn{8}{l}{\textit{Observational-only methods}} \\
PC & & obs. & 0.240 & 0.195 & 0.166 & 0.287 & 0.174 & 0.212 \\
GES & & obs. & 0.209 & 0.159 & 0.106 & 0.291 & 0.244 & 0.202 \\
FCI & & obs. & 0.036 & 0.081 & 0.126 & 0.259 & 0.133 & 0.127 \\
LiNGAM & & obs. & 0.297 & 0.405 & 0.333 & 0.379 & 0.352 & 0.353 \\
\midrule
\multicolumn{8}{l}{\textit{Interventional methods (obs.-only vs.\ w/ intervention)}} \\
\midrule
IGSP & (obs.\ only) & obs. & 0.254 & 0.377 & 0.489 & 0.428 & 0.355 & 0.381 \\
IGSP \cite{wang2017permutation} & (w/ int.) & int. & 0.472 & 0.605 & 0.618 & 0.635 & 0.480 & 0.562 \\
\midrule
UT-IGSP & (obs.\ only) & obs. & 0.321 & 0.379 & 0.478 & 0.375 & 0.353 & 0.381 \\
UT-IGSP \cite{squires2020permutation} & (w/ int.) & int. & 0.448 & 0.549 & 0.566 & 0.567 & 0.487 & 0.523 \\
\midrule
\textbf{CFM-SD(ours)} & (w/ int.) & int. & 0.829 & 0.840 & 0.596 & 0.701 & 0.790 & \textbf{0.800} \\
\bottomrule
\end{tabular}
\end{table}

\section{Full Per-$\gamma$ Results (Linear SCM)}
\label{sec:full_results}

\begin{table}[H]
\centering
\caption{Full per-$\gamma$ F1 score comparison under latent confounders ($\gamma$). Methods above the midrule use observational data only; methods below use interventional data ($n_{\mathrm{int}}=200$ per variable). Results averaged over 5 random seeds ($n_{\mathrm{obs}}=500$, $d=5$). \textbf{Bold}: best per row. CFM-SD achieves the highest average F1 across all settings with latent confounders ($\gamma > 0$). Note that we use IGSP \cite{wang2017permutation} and UT-IGSP \cite{squires2020permutation}.
}
\label{tab:synthetic_full}
\setlength{\tabcolsep}{3pt}
\small
\begin{tabular}{llrrrrrrr}
\toprule
Graph & $\gamma$ & PC & GES & FCI & LiNGAM & IGSP & UT-IGSP & \textbf{CFM-SD (ours)} \\
\midrule
\multicolumn{9}{l}{\textit{Observational-only methods \quad|\quad Interventional methods}} \\
\midrule
Fork & 0.0 & 0.000 & 0.000 & 0.000 & 0.478 & \textbf{0.978} & \textbf{0.978} & 0.956 \\
 & 0.2 & 0.240 & 0.120 & 0.000 & 0.202 & 0.444 & 0.354 & \textbf{0.833} \\
 & 0.4 & 0.267 & 0.171 & 0.000 & 0.288 & 0.344 & 0.313 & \textbf{0.662} \\
 & 0.6 & 0.280 & 0.291 & 0.000 & 0.257 & 0.344 & 0.344 & \textbf{0.606} \\
 & 0.8 & 0.411 & 0.463 & 0.181 & 0.259 & 0.251 & 0.251 & \textbf{0.573} \\
\midrule
Chain & 0.0 & 0.000 & 0.000 & 0.000 & 0.428 & \textbf{0.978} & 0.880 & \textbf{0.834} \\
 & 0.2 & 0.211 & 0.253 & 0.124 & 0.431 & \textbf{0.612} & 0.587 & \textbf{0.846} \\
 & 0.4 & 0.206 & 0.187 & 0.067 & 0.383 & 0.528 & 0.376 & \textbf{0.846} \\
 & 0.6 & 0.280 & 0.187 & 0.080 & 0.388 & 0.442 & 0.471 & \textbf{0.846} \\
 & 0.8 & 0.280 & 0.168 & 0.133 & 0.397 & 0.462 & 0.432 & \textbf{0.829} \\
\midrule
V-str. & 0.0 & 0.000 & 0.000 & 0.000 & 0.336 & \textbf{0.978} & \textbf{0.978} & 0.820 \\
 & 0.2 & 0.279 & 0.120 & 0.190 & 0.340 & 0.552 & 0.483 & \textbf{0.880} \\
 & 0.4 & 0.179 & 0.117 & 0.080 & 0.349 & 0.516 & 0.490 & \textbf{0.829} \\
 & 0.6 & 0.179 & 0.149 & 0.147 & 0.321 & 0.554 & 0.423 & \textbf{0.814} \\
 & 0.8 & 0.196 & 0.143 & 0.213 & 0.321 & 0.493 & 0.457 & \textbf{0.800} \\
\midrule
Diamond & 0.0 & 0.000 & 0.164 & 0.000 & 0.456 & \textbf{0.982} & 0.945 & 0.786 \\
 & 0.2 & 0.456 & 0.309 & 0.448 & 0.314 & 0.413 & 0.385 & \textbf{0.759} \\
 & 0.4 & 0.358 & 0.186 & 0.425 & 0.339 & 0.611 & 0.449 & \textbf{0.724} \\
 & 0.6 & 0.365 & 0.435 & 0.285 & 0.378 & 0.586 & 0.474 & \textbf{0.705} \\
 & 0.8 & 0.255 & 0.359 & 0.139 & 0.405 & 0.580 & 0.580 & \textbf{0.694} \\
\midrule
Collider & 0.0 & 0.000 & 0.200 & 0.000 & 0.375 & \textbf{1.000} & \textbf{1.000} & \textbf{1.000} \\
 & 0.2 & 0.237 & 0.394 & 0.000 & 0.330 & 0.357 & 0.357 & \textbf{0.938} \\
 & 0.4 & 0.196 & 0.227 & 0.200 & 0.311 & 0.348 & 0.348 & \textbf{0.780} \\
 & 0.6 & 0.221 & 0.200 & 0.317 & 0.429 & 0.324 & 0.355 & \textbf{0.782} \\
 & 0.8 & 0.214 & 0.200 & 0.147 & 0.314 & 0.371 & 0.375 & \textbf{0.856} \\
\midrule
\multicolumn{2}{l}{\textit{Overall average}} & 0.212 & 0.202 & 0.127 & 0.353 & 0.562 & 0.523 & \textbf{0.800} \\
\bottomrule
\end{tabular}
\end{table}

\section{CFM-SD Algorithm (Pseudocode)}
\label{sec:algorithm}

\begin{algorithm}[H]
\caption{CFM-SD: Causal Flow Matching with Simulation Data}
\label{alg:cfmsd}
\label{alg:sgcfm}
\begin{algorithmic}[1]
\REQUIRE Observational data $\mathcal{D}_{obs}$, Simulator $\mathcal{S}$,
         Intervention budget $B$, Threshold $\tau_e$
\ENSURE Estimated causal graph $\hat{\mathcal{G}}$

\STATE // \textbf{Phase 1: Conditional distribution learning (Flow Matching)}
\FOR{all variable pairs $(i, j)$ where $i \neq j$}
    \STATE Learn Flow Model $v_{\theta_{ij}}$ to approximate $P(X_j|X_i)$
\ENDFOR

\STATE // \textbf{Phase 2: Interventional data acquisition (round-robin)}
\FOR{$i = 0$ to $d-1$}
    \STATE Select values $x^*$ from percentiles of $P_{\mathrm{obs}}(X_i)$; query $\mathcal{S}.\mathrm{intervene}(i, x^*)$; store results in $\mathcal{D}_{int}^{(i)}$
\ENDFOR

\STATE // \textbf{Phase 3: Confounding detection (KDE-MMD)}
\FOR{each source $i$, target $j$ ($i \neq j$)}
    \STATE $\hat{P}(X_j|X_i{=}x) \leftarrow$ KDE of $\mathcal{D}_{obs}$ weighted by $\|X_i - x\|$
    \STATE $\hat{\Delta}_{ij} \leftarrow \mathrm{MMD}\!\left(\hat{P}(X_j|X_i{=}x),\ P(X_j|\mathrm{do}(X_i{=}x))\right)$
\ENDFOR
\STATE Adaptive threshold: $\tau_c \leftarrow \mathrm{median}(\hat{\Delta}) + \mathrm{std}(\hat{\Delta})$
\STATE $\hat{\mathcal{C}} \leftarrow \{(i,j) : \hat{\Delta}_{ij} > \tau_c\}$ \COMMENT{confounded pairs}

\STATE // \textbf{Phase 4: Causal direction via ATE}
\FOR{all variable pairs $(i, j)$ with $i < j$}
    \STATE $\mathrm{ATE}(i{\to}j) \leftarrow \mathbb{E}_{\mathcal{D}_{int}^{(i)}}[X_j] - \mathbb{E}_{\mathcal{D}_{obs}}[X_j]$
    \STATE $\mathrm{ATE}(j{\to}i) \leftarrow \mathbb{E}_{\mathcal{D}_{int}^{(j)}}[X_i] - \mathbb{E}_{\mathcal{D}_{obs}}[X_i]$
    \IF{$|\mathrm{ATE}(i{\to}j)| > \tau_e$ \textbf{or} $|\mathrm{ATE}(j{\to}i)| > \tau_e$}
        \IF{$|\mathrm{ATE}(i{\to}j)| > |\mathrm{ATE}(j{\to}i)|$}
            \STATE Add edge $i \to j$ to $\hat{\mathcal{G}}$
        \ELSE
            \STATE Add edge $j \to i$ to $\hat{\mathcal{G}}$
        \ENDIF
    \ENDIF
\ENDFOR

\STATE // \textbf{Phase 5: DAG constraint enforcement}
\WHILE{cycles exist in $\hat{\mathcal{G}}$}
    \STATE Remove the edge with the smallest $|\mathrm{ATE}|$ in the cycle
\ENDWHILE

\end{algorithmic}
\end{algorithm}

\section{Intervention Cost Comparison}
\label{sec:cost_details}

\begin{table}[H]
\centering
\caption{Intervention cost: physical simulation vs.\ wet-lab experiment. Wet-lab costs from \cite{tox21,su2018dfec}; DFT/xTB compute costs estimated from our experiments.}
\label{tab:cost_comparison}
\small
\begin{tabular}{lrrl}
\toprule
Intervention mechanism & Time/sample & Cost & Feasibility \\
\midrule
\multicolumn{4}{l}{\textit{SEI experiment (LUMO level measurement)}} \\
DFT calculation (GGA-PBE) \cite{giannozzi2009quantum} & $\sim$1--2 h/molecule & $\sim$\$1 (compute) & $\bigcirc$ \\
Cyclic voltammetry (wetlab) & 1--2 days/additive & $\sim$ \$500--2,000 & Confounded by $\mathbf{W}$ \\
Full cycle test + NMR/XPS & weeks--months & $\sim$ \$10,000+ & Confounded; destructive \\
\midrule
\multicolumn{4}{l}{\textit{QSTR experiment (electronic reactivity measurement)}} \\
xTB calculation \cite{denison2003ahr} & $\sim$1 min/molecule & $\sim$ \$0.001 & $\bigcirc$ \\
In vitro qHTS assay \cite{tox21} & hours--days/compound & $\sim$ \$50--200 & Confounded by matrix effects \\
In vivo animal study & weeks--months & $\sim$\$5,000--50,000 & Species extrapolation issues \\
\bottomrule
\end{tabular}
\end{table}

\section{Detailed Proofs of Theorems}

\subsection{Proof of Theorem~\ref{thm:intervention_identifiability} (Identifiability through Intervention)}
\label{sec:proof_identifiability}

\begin{proof}
(i) Consider an intervention $\mathrm{do}(X_i = x)$ on $X_i$.

\textbf{Case 1}: When $X_i \to X_j$ is the true causal relationship.
From the structural equation $X_j = f(X_i, \epsilon_j)$, varying $x$ also changes the distribution of $X_j$.
That is, $\exists x_1 \neq x_2$ such that 
$P(X_j|\mathrm{do}(X_i=x_1)) \neq P(X_j|\mathrm{do}(X_i=x_2))$.

\textbf{Case 2}: When $X_j \to X_i$ is the true causal relationship.
Since $X_j$ is not a child of $X_i$, intervention on $X_i$ does not affect the distribution of $X_j$ (by Assumption A3, ICM).
Therefore, $\forall x_1, x_2$: 
$P(X_j|\mathrm{do}(X_i=x_1)) = P(X_j|\mathrm{do}(X_i=x_2)) = P(X_j)$.

\textbf{Case 3}: In the case of latent confounding only ($X_i \leftarrow Z \to X_j$).
By Lemma~\ref{lem:intervention_separation}, after intervention, $Z$ and $X_i$ become independent.
Therefore, $P(X_j|\mathrm{do}(X_i=x)) = \int P(X_j|Z) P(Z) dZ = P(X_j)$, and similar to Case 2, the distribution of $X_j$ does not depend on $x$.

The same argument is applied to intervention on $X_j$.
If $X_j$ changes under intervention on $X_i$ and $X_i$ does not change under intervention on $X_j$, then $X_i \to X_j$;
if the reverse, then $X_j \to X_i$;
if neither changes, there is no causal relationship (or only latent confounding).

(ii) By applying (i) to all variable pairs, the existence and direction of all edges can be identified.
\end{proof}

\subsection{Proof of Theorem~\ref{thm:direction} (Identifiability of Causal Direction)}
\label{sec:proof_direction}

\begin{proof}
Assume that the true causal relationship is $X_i \to X_j$.
In the structural equation model, $X_j = f(X_i, \epsilon_j)$, where $\epsilon_j$ is noise independent of $X_i$.

When performing intervention $\mathrm{do}(X_i = x)$ on $X_i$, the value of $X_j$ is determined as $f(x, \epsilon_j)$.
For different intervention values $x^{(1)}, x^{(2)}, \ldots$, the ATE estimator is:
$e_{i\to j} = \mathbb{E}_{\mathcal{D}_{int}^{(i)}}[X_j] - \mathbb{E}_{\mathcal{D}_{obs}}[X_j]$.
Since $X_j = f(X_i, \epsilon_j)$ is non-trivial in $X_i$ (Assumption A3), varying $x$ changes $\mathbb{E}[f(x,\epsilon_j)]$, so $e_{i\to j} \neq 0$.

On the other hand, since $X_i$ is not a descendant of $X_j$, by the truncated factorization \cite{pearl2009causality}: $P(X_i|\mathrm{do}(X_j=y)) = P(X_i)$ for all $y$.
Therefore $e_{j\to i} = \mathbb{E}[X_i|\mathrm{do}(X_j=y)] - \mathbb{E}[X_i] = 0$.

Hence $|e_{i \to j}| > 0 = |e_{j \to i}|$ holds.
\end{proof}

\subsection{Proof of Theorem~\ref{thm:intervention_lower_bound} (Lower Bound on Required Number of Interventions)}
\label{sec:proof_lower_bound}

\begin{proof}
By contradiction. Assume that all causal structures can be identified with $d-2$ or fewer interventions.

For a DAG with $d$ variables, there exist $\binom{d}{2} = d(d-1)/2$ variable pairs, and the causal direction (or absence of causal relationship) must be determined for each pair.

With one intervention $\mathrm{do}(X_i = x)$, the presence or absence of causal effects from $X_i$ to the other $d-1$ variables can be determined.
However, information about cases where $X_i$ is not the cause (other variables $\to X_i$) cannot be obtained.

With $d-2$ interventions, at least two variables $X_a, X_b$ have not been intervened upon.
In this case, the causal direction between $X_a$ and $X_b$ cannot be directly determined.
By Theorem~\ref{thm:obs_nonidentifiable}, it is non-identifiable from observational data alone.

This is a contradiction. Therefore, at least $d-1$ interventions are required.
\end{proof}

\section{Proof of Sample Complexity Theorem}
\label{thm:sample_complexity}
\label{sec:proof_sample_complexity}

\begin{proof}
\textbf{Observational sample complexity.}
CFM-SD's Phase 1 requires learning $P(X_j|X_i)$ for all $d(d-1)$ ordered pairs.
For each pair, we perform a conditional independence test using Fisher-z statistics.
To control the family-wise error rate at level $\delta$ over all $\binom{d}{2}$ pairs, we apply Bonferroni correction with per-test significance $\delta' = \delta / \binom{d}{2}$.
The Fisher-z test achieves power $1 - \delta'$ against alternatives with effect size $\epsilon$ when
\[
n \geq \frac{z_{1-\delta'/2}^2 + z_{1-\beta}^2}{\epsilon^2} + 3 \approx \frac{2\log(d^2/\delta)}{\epsilon^2}
\]
where $z_{\alpha}$ is the $\alpha$-quantile of the standard normal.
Summing over $d(d-1)/2$ pairs and absorbing constants yields $n = O(d^2/\epsilon^2 \cdot \log(d/\delta))$.

\textbf{Interventional sample complexity.}
For Phase 4, we estimate the causal effect $e_{i \to j} = \mathrm{Corr}(\{x_i^{(k)}\}, \{y_j^{(k)}\})$ from $m$ interventional samples under $\mathrm{do}(X_i = x_i^{(k)})$.
By Hoeffding's inequality, since correlations are bounded in $[-1, 1]$:
\[
P\!\left(|\hat{e}_{i \to j} - e_{i \to j}| > \epsilon\right) \leq 2\exp(-2m\epsilon^2).
\]
Setting the right-hand side to $\delta / (d(d-1))$ (Bonferroni over all pairs) gives
$m = O(1/\epsilon^2 \cdot \log(d/\delta))$ per intervention variable.
Since we perform $d$ single-variable interventions (one per variable), the total interventional sample count is $d \cdot m = O(d/\epsilon^2 \cdot \log(d/\delta))$.
\end{proof}

\section{Proof of MMD-based Confounding Detection Theorem}
\label{thm:mmd_detection}
\label{sec:proof_mmd}

\begin{proof}
\textbf{Part (i): Consistency.}
Let $k: \mathcal{X} \times \mathcal{X} \to \mathbb{R}$ be a characteristic kernel (e.g., RBF kernel $k(x,y) = \exp(-\|x-y\|^2/2\sigma^2)$).
By definition, MMD with a characteristic kernel satisfies $\mathrm{MMD}^2(P, Q) = 0 \iff P = Q$ \cite{gretton2012kernel}.

Under Assumptions (A1)--(A5), for the pair $(X_i, X_j)$:
\begin{itemize}
\item If no latent confounding exists: by Lemma~\ref{lem:intervention_separation} and the ICM principle (A3), $P(X_j|\mathrm{do}(X_i=x)) = P(X_j|X_i=x)$ for almost all $x$, so $\mathrm{MMD}^2 = 0$.
\item If latent confounding $Z$ exists: $P(X_j|X_i=x)$ includes the backdoor path $X_i \leftarrow Z \to X_j$, while $P(X_j|\mathrm{do}(X_i=x))$ blocks this path (Lemma~\ref{lem:intervention_separation}). Hence $P(X_j|X_i=x) \neq P(X_j|\mathrm{do}(X_i=x))$ and $\mathrm{MMD}^2 > 0$.
\end{itemize}

The empirical estimator $\widehat{\mathrm{MMD}}^2$ converges to $\mathrm{MMD}^2$ in probability as $n, m \to \infty$ by the law of large numbers applied to the U-statistic formulation.

\textbf{Part (ii): Finite-sample detection power.}
The empirical MMD estimator $\widehat{\mathrm{MMD}}^2$ can be written as a U-statistic with bounded kernel:
\[
\widehat{\mathrm{MMD}}^2 = \frac{1}{n^2}\sum_{i,i'} k(x_i,x_{i'}) - \frac{2}{nm}\sum_{i,j} k(x_i,y_j) + \frac{1}{m^2}\sum_{j,j'} k(y_j,y_{j'})
\]
where $x_i \sim P_{\mathrm{obs}}$ and $y_j \sim P_{\mathrm{int}}$.
Since $k$ is bounded by $\kappa = \sup_x k(x,x)$, changing any single sample changes $\widehat{\mathrm{MMD}}^2$ by at most $2\kappa/(n \wedge m)$.
By McDiarmid's inequality:
\[
P\!\left(\widehat{\mathrm{MMD}}^2 \leq \mathrm{MMD}^2 - t\right) \leq \exp\!\left(-\frac{(n \wedge m) t^2}{2\kappa^2}\right).
\]
Setting $t = \Delta/2$ where $\Delta = \mathrm{MMD}^2(P_{\mathrm{obs}}, P_{\mathrm{int}}) > 0$:
\[
P\!\left(\widehat{\mathrm{MMD}}^2 > \frac{\Delta}{2}\right) \geq 1 - \exp\!\left(-\frac{(n \wedge m)\Delta^2}{8\kappa^2}\right).
\]
Setting threshold $\tau = \Delta/2$ completes the proof.
\end{proof}

\section{Proof of Robustness to Simulator Error Theorem}
\label{thm:simulator_robustness}
\label{sec:proof_robustness}

\begin{proof}
Let $\tilde{P}_{\mathrm{int}}$ denote the distribution generated by the simulator $S$ (with approximation error $\epsilon_{\mathrm{sim}}$) and $P_{\mathrm{int}}$ the true interventional distribution.

\textbf{Part (i).} By the kernel triangle inequality for MMD:
\[
\mathrm{MMD}(P_{\mathrm{obs}}, \tilde{P}_{\mathrm{int}}) \leq \mathrm{MMD}(P_{\mathrm{obs}}, P_{\mathrm{int}}) + \mathrm{MMD}(P_{\mathrm{int}}, \tilde{P}_{\mathrm{int}}).
\]
For the second term, since $\mathrm{TV}(P_{\mathrm{int}}, \tilde{P}_{\mathrm{int}}) \leq \epsilon_{\mathrm{sim}}$ and the kernel is bounded by $\kappa$:
\[
\mathrm{MMD}^2(P_{\mathrm{int}}, \tilde{P}_{\mathrm{int}}) \leq 2\kappa \cdot \mathrm{TV}(P_{\mathrm{int}}, \tilde{P}_{\mathrm{int}}) \leq 2\kappa \epsilon_{\mathrm{sim}}.
\]
When no confounding exists, $P_{\mathrm{obs}} = P_{\mathrm{int}}$, so $\mathrm{MMD}(P_{\mathrm{obs}}, \tilde{P}_{\mathrm{int}}) \leq \sqrt{2\kappa\epsilon_{\mathrm{sim}}} = O(\epsilon_{\mathrm{sim}}^{1/2})$. Setting the detection threshold $\tau > \sqrt{2\kappa\epsilon_{\mathrm{sim}}}$ controls the false positive rate.

\textbf{Part (ii).} The causal effect estimator is $\hat{e}_{i \to j} = \mathrm{Corr}(\{x_i^{(k)}\}, \{\tilde{y}_j^{(k)}\})$ where $\tilde{y}_j^{(k)} \sim \tilde{P}_{\mathrm{int}}$.
Since correlation is a bounded functional with Lipschitz constant $L$ with respect to the marginal distributions:
\[
|\hat{e}_{i \to j} - e_{i \to j}| \leq L \cdot \mathrm{TV}(P_{\mathrm{int}}, \tilde{P}_{\mathrm{int}}) + O(1/\sqrt{m}) \leq L\epsilon_{\mathrm{sim}} + O(1/\sqrt{m}).
\]

\textbf{Part (iii).} CFM-SD determines $X_i \to X_j$ iff $|\hat{e}_{i \to j}| > |\hat{e}_{j \to i}|$.
If the true gap is $|e_{i \to j}| - |e_{j \to i}| \geq \Delta_{\min}$, and the estimation error is bounded by $2L\epsilon_{\mathrm{sim}} + O(1/\sqrt{m}) < \Delta_{\min}/2$, then the correct direction is recovered with high probability.
\end{proof}

\section{DFT Calculation Details}
\label{sec:dft_details}
        
Quantum ESPRESSO 7.5 \cite{giannozzi2009quantum} was used for calculating HOMO/LUMO levels of electrolyte additives. Tables~\ref{tab:dft_structure} and~\ref{tab:dft_results} list molecular structures and computed energy levels.
Table~\ref{tab:dft_params} summarizes the calculation conditions.

\begin{table}[h]
\centering
\caption{DFT calculation parameters}
\label{tab:dft_params}
\small
\begin{tabular}{lp{4.5cm}}
\toprule
Parameter & Value \\
\midrule
XC functional & PBE (GGA) \cite{perdew1996generalized} \\
Pseudopotential & PAW \cite{blochl1994projector} (PSLibrary 1.0.0 \cite{dalcorso2014pseudopotentials}) \\
Wavefunction cutoff & 40--50 Ry \\
Charge density cutoff & 320--400 Ry \\
k-point sampling & $\Gamma$ point only \\
Smearing & Gaussian, 0.005--0.01 Ry \\
Structure opt. & BFGS, $10^{-4}$ Ry/Bohr \\
Cell size & 15--20 \AA{} (with vacuum) \\
\bottomrule
\end{tabular}
\end{table}

Each molecule was treated as an isolated system, with sufficient vacuum layers (approximately 7--10 \AA) to avoid interactions due to periodic boundary conditions.
Structure optimization was performed, and band energies were extracted after forces converged.
HOMO level was defined as the energy of the highest occupied band, and LUMO level as the energy of the lowest unoccupied band.

Table~\ref{tab:dft_results} summarizes the calculation results.

\begin{table}[H]
\centering
\caption{DFT-calculated energy levels of electrolyte additives}
\label{tab:dft_results}
\small
\begin{tabular}{lcccc}
\toprule
Molecule & HOMO (eV) & LUMO (eV) & Gap (eV) & F \\
\midrule
DEC & $-$6.39 & $-$0.51 & 5.88 & 0 \\
DMC & $-$6.59 & $-$0.52 & 6.07 & 0 \\
EC & $-$6.85 & $-$0.78 & 6.07 & 0 \\
FEC & $-$7.22 & $-$0.76 & 6.46 & 1 \\
VC & $-$6.14 & $-$1.14 & 5.00 & 0 \\
LiBOB & $-$6.45 & $-$1.75 & 4.70 & 0 \\
\bottomrule
\end{tabular}
\end{table}

The order of LUMO levels is LiBOB ($-$1.75 eV) $<$ VC ($-$1.14 eV) $<$ EC ($-$0.78 eV) $\approx$ FEC ($-$0.76 eV) $<$ DMC ($-$0.52 eV) $\approx$ DEC ($-$0.51 eV), which is consistent with the electrochemical knowledge that LiBOB and VC are preferentially reduced to form SEI.

Linear carbonates (DEC, DMC) have the highest LUMO levels ($-$0.51, $-$0.52 eV) and are difficult to reduce, so their contribution to SEI formation is small.
RIn actual electrolytes, these are used as solvents for low viscosity and high ionic conductivity, and do not function as SEI-forming additives.

Although FEC has almost the same LUMO level as EC, its capacity retention is significantly improved (53\% $\to$ 80\%).
This is due to the effect of F atoms changing the SEI composition (such as LiF formation), which cannot be explained by LUMO level alone.
Through multiple regression analysis, the F atom effect was estimated as $\theta_{\text{F}} = +24.2$ \%/F atom.

\end{document}